\documentclass[runningheads]{llncs}
\usepackage[T1]{fontenc}
\usepackage{graphicx}
\usepackage{amssymb}
\usepackage[nonamelimits]{amsmath}
\usepackage[pdftex,
              pdftitle={Class Prior Estimation under Covariate Shift: No Problem?},
               pdfsubject={},
               pdfkeywords={},
               pdfauthor={Dirk Tasche},
               pdfstartview=FitR,
               breaklinks=true,
               colorlinks=true]{hyperref}

\newtheorem{assumption}{Assumption}

\begin{document}
\title{Class Prior Estimation under Covariate Shift: No Problem?}
\titlerunning{Class Prior Estimation under Covariate Shift}
\author{Dirk Tasche\orcidID{0000-0002-2750-2970}}
\authorrunning{D. Tasche}
\institute{Independent researcher \\
\email{dirk.tasche@gmx.net}}
\maketitle              
\begin{abstract}
We show that in the context of classification the property of source 
and target distributions to be related by covariate shift may be lost if the 
information content captured in the covariates is reduced, for instance 
by dropping components or mapping
into a lower-dimensional or finite space. 
As a consequence, under covariate shift simple approaches to class prior estimation 
in the style of classify and count with or without adjustment are infeasible.
We prove that transformations of the covariates that preserve the covariate
shift property are necessarily sufficient in the statistical sense for the full set of covariates.
A probing algorithm as alternative approach to
class prior estimation under covariate shift is proposed.
\end{abstract}

\section{Introduction}

Class prior estimation (also known as quantification, class distribution estimation, prevalence estimation
etc.) may be considered one of the tasks referred to under the general term domain adaptation. 

Domain adaptation means
adapting algorithms designed for a source (training) dataset (also distribution or domain)
to a target (test) dataset. The source and target distributions may be different, 
a phenomenon which is called \emph{dataset shift}. In this paper, attention is
restricted to `unsupervised' domain adaptation. This term refers on the one hand to 
the situation where under the 
source distribution all events and realisations of random variables -- including 
the target (label) variable -- are observable such that in principle the
whole distribution can be estimated. On the other hand under the target distribution only the marginal 
distribution of the covariates (features) can be observed, via realisations of the covariates.
The target distribution class labels cannot be observed at all or only with delay.

Moreno-Torres et al.~\cite{MorenoTorres2012521} proposed the following popular
taxonomy of types of dataset shift:
\begin{itemize}
\item Covariate shift: Source and target posterior class probabilities are the same but 
    source and target covariate distributions may be different.
\item Prior probability shift (label shift \cite{pmlr-v80-lipton18a}, global drift \cite{hofer2013drift}):
    Source and target class-conditional covariate distributions are the same but
    source and target prior class probabilities may be different.
\item Concept shift: Source and target covariate distributions are the same but 
    source and target posterior class probabilities may be different, or 
    source and target prior class probabilities are the same but source and 
    target class-conditional covariate distributions may be different.
\item Other shift: Any dataset shift not captured by the previous types.
\end{itemize}
Covariate shift and prior probability shift are described in constructive terms. Based on
their defining properties source and target distributions are fully specified. For this reason,
a host of focussed literature is available for these two types of dataset shift. In contrast,
it is hardly possible
to make specific statements about the two other types of shift such that the literature on these types is
much more diverse and hard to capture.

In this paper we focus on covariate shift and a classification setting. We choose a
measure-theoretic approach that is particularly suitable for this context as it facilitates a rigorous joint
treatment of continuous and discrete random variables, or covariates and class labels more
specifically. We work in the same binary classification setting as 
Ben-David et al.~\cite{Ben-David2007Representations} and
Johansson et al.~\cite{pmlr-v89-johansson19a}.
Like Ben-David et al.\ and Johansson et al., 
we focus on the binary case but the results are easily generalised to the multi-class case.

Prior probability shift is robust in the following sense: If the set of covariates is transformed
in a way that reduces the information reflected by them (e.g.\ by dropping components or mapping
into a lower-dimensional or finite space) then the resulting source and target joint distributions
of covariates and labels are still related by prior probability shift. As a consequence, simple 
approaches to class prior estimation under prior probability shift can be designed which avoid the
need to estimate the full class-conditional covariate distributions.\footnote{%
The simplification may come at a cost of increased variance of the estimator (Tasche 
\cite{Tasche2021Minimising}).} The primary example for such 
an approach is the `confusion matrix method' (Gart and Buck \cite{buck1966comparison}; 
Saerens et al.~\cite{saerens2002adjusting}; 
`adjusted count' in Forman \cite{forman2005counting}). 

We show by examples and by theoretical analysis that such robustness is not displayed by covariate shift. 
Under the condition that the target distribution is absolutely continuous with respect to
the source distribution, we prove that a set of covariates passes on the covariate shift property
if and only if the transformed set of covariates 
is `sufficient' in the sense of Adragni and Cook \cite{adragni2009sufficient} and 
Tasche \cite{Tasche2022} for the untransformed set under the source distribution.
The result refines an observation of Johansson et al.~\cite{pmlr-v89-johansson19a} 
who found that covariate shift was inherited ``only if'' the 
transformation was invertible.

An important consequence of this finding is that in general for class prior estimation
under covariate shift, 
simplification in the sense of reducing
the complexity of the covariate set is not a viable path because the covariate shift property of 
identical posterior class probabilities between source and target 
distributions might get lost. We point to
a potential alternative approach, based on the so-called `probing' method of Langford and Zadrozny
\cite{langford2005estimating}.

The plan of this paper is as follows: 
We introduce the assumptions and the notation for this paper in Section~\ref{se:Setting}.
In  Section~\ref{se:loss} we give examples of how loss of information may affect the
covariate shift property. The main result (Theorem~\ref{th:main} of this paper) is presented
in Section~\ref{se:Main} while Section~\ref{se:Discuss} provides some comments on the result.
A proposal for applying `probing' 
to class prior estimation is made in Section~\ref{se:Probing}. The paper concludes 
with a short summary in Section~\ref{se:Conc}.

\section{Assumptions and Notation}
\label{se:Setting}

In this paper, we work only at population (distribution) level as this level is appropriate for the design of
estimators and predictors as well as the study of their fundamental properties. A detailed 
treatment of the intricacies of sample properties is not needed.

We follow the example of Scott \cite{Scott2019} who introduced consistent concepts and notation
for appropriately dealing with the classification setting we need. As the concept of information
plays a more important role in this paper than in Scott's, we dive somewhat deeper into
the measure-theoretic details of the setting than Scott.

\subsection{Setting for Binary Classification in the Presence of Dataset Shift}
\label{se:measure}

We introduce a measure-theoretic setting, expanding the setting of Scott \cite{Scott2019} and adapting 
the approach of Holzmann and Eulert \cite{HolzmannInformation2014} and
Tasche \cite{Tasche2022}. Phrasing the context in measure theory terms is particularly efficient when random variables 
with continuous and discrete distributions are studied together like in the case of binary or 
multi-class classification. Moreover, the measure-theoretic notion of $\sigma$-algebras allows for the 
convenient description of differences in available information. 

We use the following population-level description of the binary classification problem in terms
of measure theory.
See standard textbooks on probability theory like Billingsley \cite{billingsley1986probability} or 
Klenke \cite{klenke2013probability} for
formal definitions and background of the notions introduced in Assumption~\ref{as:setting}.
\begin{assumption}\label{as:setting}
$(\Omega, \mathcal{A})$ is a measurable space. The \emph{source distribution} $P$ and
the \emph{target distribution} $Q$ are probability measures on $(\Omega, \mathcal{A})$.
An event $A_1 \in \mathcal{A}$ with $0 < P[A_1] < 1$ and 
a sub-$\sigma$-algebra $\mathcal{H} \subset \mathcal{A}$ with $A_1 \notin \mathcal{H}$ are fixed.
$A_0 = \Omega \setminus A_1$ is the complementary event of $A_1$ in $\Omega$.
\end{assumption}

In the literature, $P$ is also called `training distribution' while $Q$ is also referred to
as `test distribution'.

\emph{Interpretation.}
The elements $\omega$ of $\Omega$ are objects (or instances) with class (label) and covariate
(feature) attributes. 
$\omega \in A_1$ means that $\omega$ belongs to class 1 (or the positive class). 
$\omega \in A_0$
means that $\omega$ belongs to class 0 (or the negative class).

The $\sigma$-algebra $\mathcal{A}$ of events $M \in \mathcal{A}$ is a collection of subsets $M$
of $\Omega$ with the property that they can be assigned probabilities $P[M]$ and $Q[M]$ in a logically consistent way.
In the literature, thanks to their role of reflecting the available information, $\sigma$-algebras are 
sometimes also called \emph{information set} (Holzmann and Eulert \cite{HolzmannInformation2014}). In the following,
we use both terms exchangeably.

\emph{Binary classification problem.} The sub-$\sigma$-algebra $\mathcal{H} \subset \mathcal{A}$ 
contains the events which are observable
at the time when the class label of an object $\omega$ has to be predicted. Since $A_1 \notin \mathcal{H}$, then
the class of an object may not yet be known. It can only be predicted on the basis of the events $H \in \mathcal{H}$
which are assumed to reflect the features of the object.

\emph{Dataset shift.} We denote by $\mathcal{H}_A$ the minimal sub-$\sigma$-algebra of $\mathcal{A}$ containing
both $\mathcal{H}$ and
$\sigma(\{A_1\}) = \{\emptyset, A_1, A_0, \Omega\}$, i.e. $\mathcal{H}_A =
\sigma\bigl(\mathcal{H} \cup \sigma(\{A_1\})\bigr)$.
The $\sigma$-algebra $\mathcal{H}_A$ can be represented as 
\begin{equation}\label{eq:HA}
\mathcal{H}_A \ =\
\bigl\{(A_1 \cap H_1) \cup (A_0 \cap H_0): H_1, H_0\in \mathcal{H}\bigr\}.
\end{equation}
A standard assumption in machine learning is that source and target distribution are the
same, i.e.\ $P = Q$.
The situation where $P[M] \neq Q[M]$ holds for at least one $M\in \mathcal{H}_A$ is called
\emph{dataset shift} (Moreno-Torres et al.~\cite{MorenoTorres2012521}, Definition~1).

\emph{Class prior estimation}. Under dataset shift as defined above, typically the  
prior probabilities $P[A_1]$ of the positive class in the source distribution 
(assumed to be observable) and $Q[A_1]$ in
the target distribution (assumed to be unknown or known with delay only) are different.
Class prior estimation in the binary classification context of Assumption~\ref{as:setting}
is the task to estimate $Q[A_1]$, based on observations from $P$ (the entire source distribution)
and from\footnote{%
$Q|\mathcal{H}$ stands for the measure $Q$ with domain restricted to $\mathcal{H}$.} 
$Q|\mathcal{H}$ (the target distribution of the covariates, also called features).

\emph{Notation.} 
Denote by $\mathbf{1}_M$ the indicator function of an event $M$, i.e.\ $\mathbf{1}_M(\omega)=1$
if $\omega\in M$ and $\mathbf{1}_M(\omega)=0$ if $\omega \notin M$.

If $X$ is a real-valued random variable on a probability space $(\Omega, \mathcal{F}, P)$ 
and $\mathcal{G}\subset\mathcal{F}$ is a sub-$\sigma$-algebra of $\mathcal{F}$, then
a random variable $\Psi$ is called \emph{expectation of $X$ conditional on $\mathcal{G}$}
(see, e.g., Definition~8.11 of Klenke \cite{klenke2013probability})
if it has the following two properties:\\
(i) $\Psi$ is $\mathcal{G}$-measurable.\\
(ii) For all events $G \in \mathcal{G}$ it holds that 
    $E_P[\mathbf{1}_G\,X] = E_P[\mathbf{1}_G\,\Psi]$.
    
In the following, we use the usual shorthand notation $\Psi = E_P[X\,|\,\mathcal{G}]$.
In the case of an indicator function of an event $F \in \mathcal{F}$, the conditional
expectation $E_P[\mathbf{1}_F\,|\,\mathcal{G}]$ is called \emph{probability of $F$ conditional on $\mathcal{G}$}
and denoted by $P[F\,|\,\mathcal{G}]$.

\subsection{Reconciliation of Machine Learning and Measure Theory Settings}

The  setting of Assumption~\ref{as:setting} is similar
to a standard setting for binary classification in the machine learning and pattern 
recognition literature (see e.g.~Scott \cite{Scott2019} or Devroye et al.~\cite{devroye1996probabilistic}): 

Typically a random vector $(X, Y)$ is studied, where $X$ stands for the covariates of an object 
and $Y$ stands for its class. $X$ is assumed to take values in a feature space $\mathfrak{X}$ (often
$\mathfrak{X} = \mathbb{R}^d$) while 
$Y$ takes either the value 0 (or $-1$) or the value 1 (for the positive class).

Standard formulation of the binary classification problem: 
Predict the value of $Y$ from $X$ or make
an informed decision on the occurence or non-occurrence of the event $Y=1$ despite only being able 
to observe the values of $X$.

This is captured by the measure-theoretic setting of Assumption~\ref{as:setting}: 
Assume that $X$ and $Y$ map $\Omega$ into $\mathfrak{X}$ and $\{0,1\}$ respectively. 
Choose $\mathcal{H} = \sigma(X)$ (the
smallest sub-$\sigma$-algebra of $\mathcal{A}$ such that $X$ is measurable) and 
$A_1 = \{Y=1\} = \{\omega\in\Omega: Y(\omega)=1\}$.

In many machine learning papers, the image (or pushforward) measure of $P$ (or $Q$ if it refers to the 
target distribution) under the mapping $(X,Y)$ (see Definition~1.98 of Klenke \cite{klenke2013probability})
is denoted by $p(x,y)$.

Often no probability space is specified but only samples $(x_1, y_1), \ldots, (x_n, y_n)$ of
realisations of $(X,Y)$ from the source distribution and $x_{n+1}, \ldots, x_{n+m}$ of realisations 
of $X$ from the
target distribution are assumed to be given. This context is sometimes called `unsupervised domain adaptation'.
Usually the samples are assumed to have been generated through i.i.d.\ drawings from
some population distributions which may be identified with $(\Omega, \mathcal{A}, P)$ 
and $(\Omega, \mathcal{A}, Q)$ as described above.

\subsection{More on Dataset Shift}
\label{se:more}

Arguably, the two most important special cases of dataset shift are the following, in the terms introduced in Assumption~\ref{as:setting}:
\begin{itemize}
\item \emph{Covariate shift} (Moreno-Torres et al.~\cite{MorenoTorres2012521}, Definition~3; Storkey
\cite{storkey2009training}, Section~5):\\ 
In this case, $P|\mathcal{H} \neq Q|\mathcal{H}$ holds but $P[A_1\,|\,\mathcal{H}] = 
Q[A_1\,|\,\mathcal{H}]$, i.e.\ the posterior probabilities  under $P$ and $Q$ are the same but
the covariate distributions may be different.
\item \emph{Prior probability shift} (Moreno-Torres et al.~\cite{MorenoTorres2012521}, Definition~2;
Storkey \cite{storkey2009training},
Section~6):\\ In this case, we have $P[A_1] \neq Q[A_1]$ but 
$P[H\,|\,A_i] = Q[H\,|\,A_i]$, $i\in\{0,1\}$, for all $H\in \mathcal{H}$, i.e.\
the class-conditional covariate source and target distributions are the same but the unconditional
class prior probabilities may be different.
\end{itemize}
Covariate shift and prior probability shift are similar in the sense that in both cases
one of the conditional distributions (of $A_1$ conditional on $\mathcal{H}$ and of $\mathcal{H}$ conditional
on $\sigma(\{A_1\})$ respectively) are invariant between $P$ and $Q$, and at least one pair 
of the marginal distributions
(of $\mathcal{H}$ and $\sigma(\{A_1\})$ respectively) are different.

Thanks to the invariance assumptions on the conditional distributions in prior probability shift and in covariate 
shift, these two types of dataset shift are relatively easily amenable to mathematical treatment and, 
therefore, have
received considerable attention by researchers. See e.g.\ Qui{\~n}onero-Candela et al.~\cite{quinonero2008dataset} 
for covariate shift and Caelen \cite{caelen2017quantification} for prior probability shift, as well as
the references therein.

Note that the definition of dataset shift in Section~\ref{se:measure}
explicitly mentions an associated set of covariates (features, 
represented through the sub-$\sigma$-algebra $\mathcal{H}$). In Section~\ref{se:loss}, 
we are going to look closer at the
question whether or not the covariate shift property is preserved in the relationship between
source and target distribution if the amount of information reflected by the set of covariates
is reduced. Formally, the question is phrased as follows:

\emph{Under Assumption~\ref{as:setting}, if $\mathcal{G} \subset \mathcal{H}$ is another sub-$\sigma$-algebra
of $\mathcal{A}$, does then $P[A_1\,|\,\mathcal{H}]$ $= Q[A_1\,|\,\mathcal{H}]$ imply
$P[A_1\,|\,\mathcal{G}]= Q[A_1\,|\,\mathcal{G}]$?}

\section{Covariate Shift is Fragile}
\label{se:loss}

In theory, class prior estimation under covariate shift is straightforward. Assume
that the source distribution $P$ and the target distribution $Q$ are related through 
covariate shift as defined in Section~\ref{se:more}. Then by the law of total probability and the fact
that $P[A_1\,|\,\mathcal{H}] = Q[A_1\,|\,\mathcal{H}]$, the prior class probability $Q[A_1]$ of the
positive class can be represented as
\begin{equation}\label{eq:PCC}
Q[A_1] \ = \ E_Q\bigl[P[A_1\,|\,\mathcal{H}]\bigr].
\end{equation}
As mentioned in Section~\ref{se:measure}, both $P[A_1\,|\,\mathcal{H}]$ and 
$Q[A_1\,|\,\mathcal{H}]$ typically are observable at the time when $Q[A_1]$ is to be estimated
such that $Q[A_1]$ in principle can be calculated by means of \eqref{eq:PCC}. In the literature
on class prior estimation, the approach based on \eqref{eq:PCC} is known as `probability estimation 
\& average (P \& A)' (Bella et al.~\cite{bella2010quantification}) or 
`probabilistic classify \& count (PCC)' (Gonz\'alez et al.~\cite{Gonzalez:2017:RQL:3145473.3117807}).

Unfortunately, \eqref{eq:PCC} may not work well in practice:
\begin{itemize}
\item Card and Smith~\cite{card2018importance} observed that poor calibration of the estimates of the
posterior class probabilities would entail poor results for the PCC prior probability estimates.
\item At a more fundamental level, Storkey (\cite{storkey2009training}, Section~5.1) 
pointed out that the probability masses of the covariates might be quite differently located
under the source and target distributions. As a consequence, an estimate of  $P[A_1\,|\,\mathcal{H}]$
made under the source distribution $P$ might turn out to be rather biased in those regions
of the covariate space to which the target distribution $Q$ attributes most mass.  
This problem can be mitigated by `importance weighting' which, however, may significantly complicate
the estimation procedure.
\end{itemize}
Due to these issues, it is tempting to try to avoid the potentially difficult estimation
of the posterior class probability $P[A_1\,|\,\mathcal{H}]$ which is conditioned on the full
covariate information set $\mathcal{H}$, by mimicking the simplification achieved through the
confusion matrix method (Saerens et al.~\cite{saerens2002adjusting}; also called `adjusted count'
in Forman \cite{forman2005counting}) under prior probability shift.

Adapting the confusion matrix method to covariate shift would work as follows:
Fix some hard (i.e.\ taking either the value $0$ or the value $1$) 
classifier which is a function of the covariates and therefore $\mathcal{H}$-measurable.
We can identify the classifier with an event $H \in \mathcal{H}$ which specifies 
the range of the covariates on which a positive class label is predicted.
If the source distribution $P$ and the target distribution $Q$ are related by covariate shift
for the simple information set  $\mathcal{G} = \{\emptyset, H, \Omega\setminus H, \Omega\} \subset 
\mathcal{H}$ then
the following special case of \eqref{eq:PCC} applies:
\begin{equation}
Q[A_1] \ \overset{\text{\large{?}}}{=} 
\ Q[H]\,P[A_1\,|\,H] + (1-Q[H])\,P[A_1\,|\,(\Omega\setminus H)].\label{eq:simple}
\end{equation}
Eq.~\eqref{eq:simple} appears to suggest a simple and efficient approach to class prior estimation
under covariate shift which avoids the potentially difficult problem to estimate 
$P[A_1\,|\,\mathcal{H}]$ for the more complex information set $\mathcal{H}$.

But can we always find a classifier (observable event) $H$ such that the following condition for covariate shift 
with respect to $\mathcal{G} = \{\emptyset, H, \Omega\setminus H, \Omega\}$ and,
as a consequence, also \eqref{eq:simple} hold true? 
\begin{equation}\label{eq:simpleShift}
Q[A_1\,|\,H] = P[A_1\,|\,H]\quad \text{and}\quad 
Q[A_1\,|\,(\Omega\setminus H)] = P[A_1\,|\,(\Omega\setminus H)].
\end{equation}
The question mark in \eqref{eq:simple} is meant to suggest that the answer is `no'. This
is illustrated with the following example.

\begin{example}\label{ex:Simple}
We revisit the binormal model with equal variances as 
an example that fits into the setting of Assumption~\ref{as:setting}.
The source distribution $P$ is defined by specifying 
the marginal distribution of $Y = \begin{cases}
1, & \text{on}\ A_1\\
0, & \text{on}\ A_0
\end{cases}$, with $\mathrm{P}[A_1] = p \in (0,1)$, 
     and defining the class-conditional distributions of the covariate $X$ given $Y$ 
     as normal distributions with equal variances:
\begin{equation}\label{eq:CondNormal}
    P[X \in \cdot\,|\,A_1]  = \mathcal{N}(\nu, \sigma^2) \quad \text{and}\quad
    P[X \in \cdot\,|\,A_0]  = \mathcal{N}(\mu, \sigma^2). 
\end{equation}
    In \eqref{eq:CondNormal}, we assume that $\mu < \nu$ and $\sigma > 0$. The unconditional
    distribution of $X$ then is a mixture with weight $p$ of the two normal distributions.
    
The posterior class probability $P[A_1\,|\,\mathcal{H}]$ for $\mathcal{H} = \sigma(X)$ in this setting is given by
$P[A_1\,|\,\mathcal{H}]  =  \bigl(1 + \exp(a\,X + b)\bigr)^{-1}$,
with $a = \frac{\mu-\nu}{\sigma^2} < 0$ and $b = \frac{\nu^2-\mu^2}{2\,\sigma^2} + \log\left(\frac{1-p}{p}\right)$.
For the target distribution $Q$, we only specify the marginal distribution of the covariate $X$ 
as another normal
distribution with mean $E_Q[X] = \tau$ and variance $\mathrm{var}_Q[X] = \sigma^2 + p\,(1-p)\,(\mu -\nu)^2$ such 
that the variance of $X$ under $Q$ matches the variance of $X$ under $P$.\\
Under covariate shift, then by \eqref{eq:PCC} it holds for the target prior class probability $Q[A_1]$ that
\begin{equation}\label{eq:total}
Q[A_1] \ = \ E_Q\left[\bigl(1 + \exp(a\,X + b)\bigr)^{-1}\right].
\end{equation}
To illustrate the effect of simplification as suggested by \eqref{eq:simple}, we define 
a family of classifiers $H_x \ = \ \{X > x\}$ for thresholds $x \in \mathbb{R}$. 

Figure~\ref{fig:2} shows the true target prior class probability 
$Q[A_1]$ according to \eqref{eq:total} (constant, dashed line) 
and, for moving threshold $x$, `pseudo' priors according to \eqref{eq:simple} (solid curve). 
As the pseudo priors do
not match the true prior, the covariate shift property \eqref{eq:simpleShift} 
must be violated for all information sets 
\begin{equation*}
\mathcal{G}_x = \{\emptyset, H_x, \Omega \setminus H_x, \Omega\} \subset \mathcal{H} = 
\sigma(X).
\end{equation*}
 This is due to the loss of information compared to the full information set
$\mathcal{H}$ associated with the covariate $X$. \qed
\end{example}
\begin{figure}[tb]
\begin{center}
\includegraphics[width=9cm]{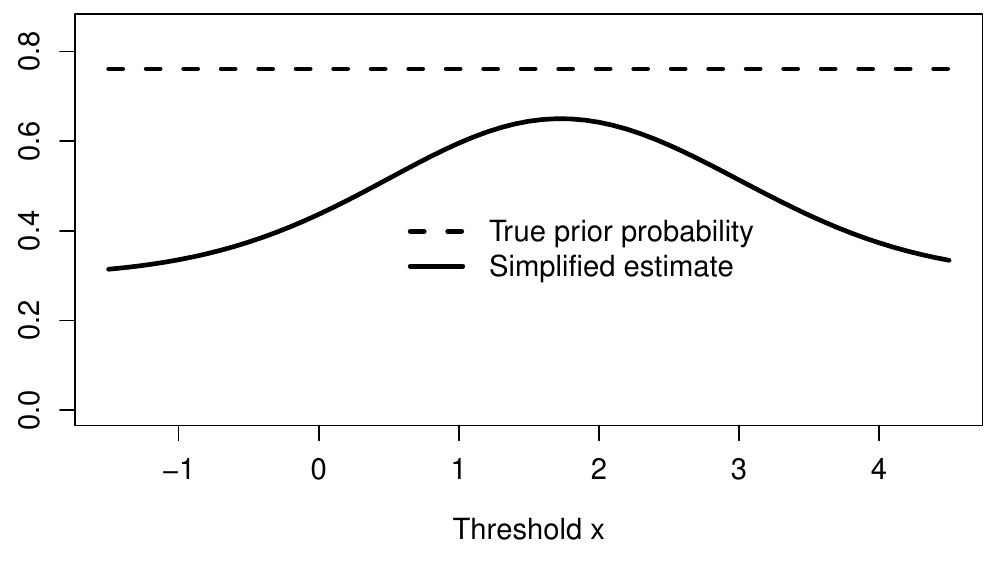}
\caption{Illustration of Example~\ref{ex:Simple}, with parameters $\mu=0$, $\nu=1.5$, 
$\sigma=1$, $p=0.3$ and $\tau=2.5$ for the normal distributions and the mixture parameter $p$
of the source distribution. Simplification of the covariate information set causes the covariate shift property
to be lost as demonstrated by the failure to hit the true target prior $Q[A_1]$ with the simplified estimates
according to \eqref{eq:simple}.
} \label{fig:2}
\end{center}
\end{figure}
On the basis of Example~\ref{ex:Simple}, we can conclude that
under Assumption~\ref{as:setting}, if $\mathcal{G} \subset \mathcal{H}$ is another sub-$\sigma$-algebra
of $\mathcal{A}$, then $P[A_1\,|\,\mathcal{H}] = Q[A_1\,|\,\mathcal{H}]$ does not always imply
$P[A_1\,|\,\mathcal{G}]= Q[A_1\,|\,\mathcal{G}]$, i.e.\ the covariate property may get lost if
the amount of information represented by the covariates is reduced.

Information loss and subsequent loss of the covariate shift property can also be the consequence of
deploying `domain-invariant representations' (Johansson et al.~\cite{pmlr-v89-johansson19a}, Section~4.1).

\section{Covariate Shift and Statistical Sufficiency}
\label{se:Main}

In the following we will identify sufficient and necessary conditions for simplifications 
of covariate shift like
\eqref{eq:simpleShift} to hold. We will see that indeed it is almost impossible for 
\eqref{eq:simpleShift} to be true if the information set $\mathcal{H}$ of Assumption~\ref{as:setting}
is large compared to the information set $\mathcal{G}$ on which \eqref{eq:simpleShift} is based.
\begin{definition}\label{de:CA}
Under Assumption~\ref{as:setting}, 
denote by $C_A(P, \mathcal{H})$ the set of 
all probability measures $Q$ on $(\Omega, \mathcal{A})$ such that $P$ and $Q$ are related 
by covariate shift, i.e.
\begin{equation*}
C_A(P, \mathcal{H}) \ = \ \bigl\{Q\ \text{probability measure on}\ (\Omega, \mathcal{A}): \
    P[A_1\,|\,\mathcal{H}] = Q[A_1\,|\,\mathcal{H}]\bigr\}.
\end{equation*}
Denote by $C^\ast_A(P, \mathcal{H})$ the set of 
all probability measures $Q$ on $(\Omega, \mathcal{A})$ such that $P$ and $Q$ are related 
by covariate shift and $Q$ is absolutely continuous\footnote{%
$Q$ is absolutely continuous with respect to $P$ on $\mathcal{H}$ (expressed symbolically
as $Q|\mathcal{H}\ll P|\mathcal{H}$) if $P[N] = 0$ for $N\in\mathcal{H}$ implies
$Q[N] = 0$. By the Radon-Nikodym theorem (Klenke \cite{klenke2013probability}, Corollary~7.34) 
then there exists an $\mathcal{H}$-measurable 
non-negative function $h$ such that $Q[H] = E_P[h\,\mathbf{1}_H]$ for all
$H\in\mathcal{H}$. The function $h$ is called density of $Q$ with respect to $P$.} 
with respect to $P$ on $\mathcal{H}$, i.e.
\begin{equation*}
C^\ast_A(P, \mathcal{H}) \ = \ \bigl\{Q\in C_A(P, \mathcal{H}): Q|\mathcal{H}\ll P|\mathcal{H}\bigr\}.
\end{equation*}
\end{definition}
The following example illustrates the definition of $C_A(P, \mathcal{H})$ in some simple
special cases.
\begin{example}\label{ex:counter}
Consider the following three special cases for $\mathcal{H}$:
\begin{enumerate}
\item[(i)]\label{it:i} If $A_1 \in \mathcal{H}$  then we have that 
$\mathcal{H} \supset \sigma(\{A\})$ and 
$$C_A(P, \mathcal{H})\ =\ 
\{\text{All probability measures on}\ (\Omega, \mathcal{A})\},$$ 
because in this case it holds that
$P[A_1\,|\,\mathcal{H}] = \mathbf{1}_{A_1}  =  Q[A_1\,|\,\mathcal{H}]$.
\item[(ii)] If $A_1$ and $\mathcal{H}$ are independent under $P$ and $Q$, it follows that
\begin{equation*}
P[A_1\,|\,\mathcal{H}]\ =\ P[A_1] \quad \text{and}\quad Q[A_1\,|\,\mathcal{H}]\ =\ Q[A_1].
\end{equation*}
Hence we have $Q \in C_A(P, \mathcal{H})$  if and only if $P[A_1] = Q[A_1]$.
\item[(iii)] If $\mathcal{H} = \{\emptyset, \Omega\}$ we are in a special case of (ii).
This implies 
\begin{equation*}
C_A(P, \mathcal{H}) \ = \ \bigl\{Q\ \text{probability measure on}\ (\Omega, \mathcal{A}): \
    P[A_1] = Q[A_1]\bigr\}. \qquad \qed
\end{equation*}
\end{enumerate}
\end{example}
Note that in case~(i) of Example~\ref{ex:counter}, $P[A_1] \neq Q[A_1]$ is possible.

\begin{remark}\label{rm:exclude}
The set $C_A(P, \mathcal{H})$ contains all probability measures $Q$ with the property that
there is an event $\Omega_Q \in \mathcal{H}$ such that $Q[\Omega_Q] =1$ and  $P[\Omega_Q] = 0$, i.e.\
$Q$ and $P$ are mutually singular. Although in this case there is an $\mathcal{H}$-measurable random variable
$\Psi$ that is both a version of $P[A_1\,|\,\mathcal{H}]$ and of $Q[A_1\,|\,\mathcal{H}]$, it is impossible
to completely learn $\Psi$ from the source distribution $P$ because no instances $\omega \in \Omega_Q$
can be sampled under $P$ due to $P[\Omega_Q] = 0$. Hence the distributions $Q$ which
are singular to $P$ are not of great theoretical interest. 
$C_A(P, \mathcal{H})$ may also contain probability measures $Q$ with both absolutely continuous and singular
components (with respect to $P$). In this case, it is not possible to completely learn $\Psi$ from $P$ either.
Therefore, in the following the focus is on the distributions $Q$ that are absolutely continuous with respect to $P$. 
\qed
\end{remark}

At first glance, one might guess that the tower property of conditional expectations 
(Klenke \cite{klenke2013probability}, Theorem~8.14) implies
$C_A(P, \mathcal{H}) \subset C_A(P, \mathcal{G})$ if $\mathcal{G}$ is a sub-$\sigma$-algebra of $\mathcal{H}$.
However, the following example shows that this is not true in general.

\begin{example}\label{ex:ind}
Assume that $\mathcal{H} = \sigma(\mathcal{F} \cup \mathcal{G})$ for sub-$\sigma$-algebras $\mathcal{F}$, 
$\mathcal{G}$ of $\mathcal{A}$, with $A_1\notin \mathcal{H}$. Assume further
that $\mathcal{G}$ and $\sigma\bigl(\sigma(\{A_1\}) \cup \mathcal{F}\bigr)$ are independent 
under $P$ and $Q$. Then it follows that
$P[A_1\,|\,\mathcal{H}]  = P[A_1\,|\,\mathcal{F}]$ and
$Q[A_1\,|\,\mathcal{H}]  = Q[A_1\,|\,\mathcal{F}]$.

Hence we have $Q \in C_A(P, \mathcal{H})$ if and only if $Q\in C_A(P, \mathcal{F})$. By case~(ii) of
Example~\ref{ex:counter}, we have $Q \in C_A(P, \mathcal{G})$ if and only if $P[A_1] = Q[A_1]$. Hence,
if there is a $Q\in C_A(P, \mathcal{F})$ with $P[A_1] \neq Q[A_1]$, we have an example showing
that $C_A(P, \mathcal{H}) \not\subset C_A(P, \mathcal{G})$ may happen despite $\mathcal{G} \subset \mathcal{H}$. 
\qed
\end{example}
Example~\ref{ex:ind} demonstrates that the covariate shift property may get lost if components of the covariates
are dropped. We continue with presenting sufficient criteria for covariate shift (Lemma~\ref{le:basic})
and inheritance of covariate shift (Proposition~\ref{pr:inherit} below). 

\begin{lemma}\label{le:basic}
Under Assumption~\ref{as:setting}, 
assume further that $Q$ is absolutely continuous with respect to $P$ on $\mathcal{H}_A$ and that
there is an $\mathcal{H}$-measurable density $h$ of $Q|\mathcal{H}_A$ with respect to $P|\mathcal{H}_A$.
Then it follows that $Q \in C^\ast_A(P, \mathcal{H})$.
\end{lemma}
\begin{proof}
Fix any $H\in \mathcal{H}$. Then we obtain that
\begin{multline*}
E_Q\bigl[\mathbf{1}_H\,P[A_1\,|\,\mathcal{H}]\bigr]  = E_P\bigl[h\,\mathbf{1}_H\,P[A_1\,|\,\mathcal{H}]\bigr]\\
     = E_P\bigl[E_P[h\,\mathbf{1}_{A_1\cap H}\,|\,\mathcal{H}]\bigr]
     = E_P[h\,\mathbf{1}_{A_1\cap H}]
     = Q[A_1 \cap  H].
\end{multline*}
This implies $P[A_1\,|\,\mathcal{H}] = Q[A_1\,|\,\mathcal{H}]$. \qed
\end{proof}

We are now going to point out connections between the notion of
covariate shift and the following two concepts that have been considered in the literature in
other contexts:
\begin{itemize}
\item Covariate shift with posterior drift (Scott \cite{Scott2019}): Under Assumption~\ref{as:setting}, 
dataset shift between $P$ and $Q$
is more specifically called \emph{covariate shift with posterior drift} if there is 
an increasing function $f:[0,1] \to \mathbb{R}$ such that it holds that
\begin{equation}\label{eq:CSPD}
Q[A_1\,|\,\mathcal{H}] \ = \ f\bigl(P[A_1\,|\,\mathcal{H}]\bigr).
\end{equation}
\item Sufficiency \cite{devroye1996probabilistic,adragni2009sufficient,Tasche2022}:
Under Assumption~\ref{as:setting}, if
$\mathcal{G} \subset \mathcal{H}$ is another sub-$\sigma$-algebra of $\mathcal{A}$ then $\mathcal{G}$ is
called (statistically) \emph{sufficient} for $\mathcal{H}$ with respect to $A_1$ 
if $P[A_1\,|\,\mathcal{G}] = P[A_1\,|\,\mathcal{H}]$ holds true.
\end{itemize}

\begin{proposition}\label{pr:inherit}
Under Assumption~\ref{as:setting}, let $P$ and $Q$ be related by covariate shift with posterior drift
such that \eqref{eq:CSPD} holds for an increasing function $f$. Assume that 
$\mathcal{G}$ is sufficient for $\mathcal{H}$ with respect to $A_1$ under the source
distribution $P$ and that $Q$ is absolutely
continuous with respect to $P$ on $\mathcal{H}$. Then 
$Q[A_1\,|\,\mathcal{G}]  =  f\bigl(P[A_1\,|\,\mathcal{G}]\bigr)$ follows.
\end{proposition}
\begin{proof}
Let $h \ge 0$ be a density of $Q$ with respect to $P$ on $\mathcal{H}$. Then, in particular, $h$ is
$\mathcal{H}$-measurable. For any $G\in \mathcal{G}$, we therefore obtain
\begin{multline*}
E_Q\bigl[\mathbf{1}_G\,f\bigl(P[A_1\,|\,\mathcal{G}]\bigr)\bigr]  = 
    E_P\bigl[h\,\mathbf{1}_G\,f\bigl(P[A_1\,|\,\mathcal{G}]\bigr)\bigr]
    = E_P\bigl[h\,\mathbf{1}_G\,f\bigl(P[A_1\,|\,\mathcal{H}]\bigr)\bigr]\\
    = E_Q\bigl[\mathbf{1}_G\,f\bigl(P[A_1\,|\,\mathcal{H}]\bigr)\bigr]
    = E_Q\bigl[\mathbf{1}_G\,Q[A_1\,|\,\mathcal{H}]\bigr] 
    = Q[A_1 \cap G].
\end{multline*}
This implies the assertion.\qed
\end{proof}

Based on Lemma~\ref{le:basic} and Proposition~\ref{pr:inherit}, we are in a position to prove
the main result of this paper. It states that an information subset inherits the covariate shift property
from its information superset for all absolutely continuous target distributions if and only if the subset is
statistically sufficient for the superset with
respect to the positive class label under the source distribution. 
\begin{theorem}\label{th:main}
Under Assumption~\ref{as:setting}, let $\mathcal{G} \subset \mathcal{H}$ be another sub-$\sigma$-algebra
of $\mathcal{A}$. Then $\mathcal{G}$ is sufficient for $\mathcal{H}$ with respect to $A_1$ 
under the source distribution $P$ if and only
if  $C_A^\ast(P, \mathcal{H}) \subset C_A^\ast(P, \mathcal{G})$ holds true.
\end{theorem}
\begin{proof}
The `only if' part of the assertion is implied by Proposition~\ref{pr:inherit}. By the definition 
of conditional probability, for the `if' part 
we have to show that for each $H\in\mathcal{H}$ it holds that 
$P[A_1\cap H ] = E_P\bigl[\mathbf{1}_H\,P[A_1\,|\,\mathcal{G}]\bigr]$. This is obvious for $H$ with
$P[H] = 0$. Hence fix an event $H\in \mathcal{H}$ and assume $P[H] > 0$.

Define the probability measure $Q_H$ on $(\Omega, \mathcal{A})$ as $P$ conditional on $H$, i.e.
\begin{equation*}
Q_H[M] \ = \ P[M\,|\,H] \ = \ \frac{P[M\cap H]}{P[H]}, \quad \text{for all}\ M \in \mathcal{A}.
\end{equation*}
This $Q_H$ is absolutely continuous with respect to $P$ on $\mathcal{A} \supset \mathcal{H_A}$,
with $\mathcal{H}$-measurable density $\frac{\mathbf{1}_H}{P[H]}$. Hence, by Lemma~\ref{le:basic} we
obtain $Q_H \in C^\ast_A(P, \mathcal{H})$. By assumption, this implies 
$Q_H \in C^\ast_A(P, \mathcal{G})$, and in particular $P[A_1\,|\,\mathcal{G}] = Q_H[A_1\,|\,\mathcal{G}]$.
From this, it follows that
\begin{multline*}
E_P\bigl[\mathbf{1}_H\,P[A_1\,|\,\mathcal{G}]\bigr]  = 
    P[H]\,E_{Q_H}\bigl[P[A_1\,|\,\mathcal{G}]\bigr]\\
    = P[H]\,E_{Q_H}\bigl[Q_H[A_1\,|\,\mathcal{G}]\bigr]
     = P[H]\,Q_H[A_1] 
     = P[A_1\cap H].
\end{multline*}
This completes the proof.\qed
\end{proof}

\section{Discussion of Theorem~\ref{th:main}}
\label{se:Discuss}

Can sufficiency of $\mathcal{G}$ for $\mathcal{H}$ with respect to $A_1$ be characterised in
other ways than just requiring $P[A_1\,|\,\mathcal{G}] = P[A_1\,|\,\mathcal{H}]$?
\begin{itemize}
\item As observed by Devroye et al.~\cite{devroye1996probabilistic} (Section~32), if $\mathcal{G} = \sigma(T)$
is generated by some random variable $T$, then $\mathcal{G}$ is sufficient for $\mathcal{H}$ if and only if
there exists a measurable function $g$ such that $P[A_1\,|\,\mathcal{H}] = g(T)$.
\item Primary examples for such $T$ are transformations $T = f(P[A_1\,|\,\mathcal{H}])$ of the
posterior class probability which 
may emerge as scoring classifiers optimising the area under the Receiver Operating Characteristic (ROC)
or the area under the Brier curve (Tasche \cite{Tasche2022}, Section~5.3). The process to 
reengineer $P[A_1\,|\,\mathcal{H}]$ from $T$ is called `calibration' (see Kull et 
al.~\cite{kull2017betacalibration} and the references therein).
\end{itemize}

Johansson et al.~\cite{pmlr-v89-johansson19a} wrote in Section~4.1: 
``One interpretation $\ldots$ is that covariate
shift ([their] Assumption 1) need not hold with respect to the
representation $Z = \phi(X)$, even if it does with respect
to $X$. With $\phi^{-1}(z) = \{x : \phi(x) = z\}$,
\begin{equation}\label{eq:Johan}
p_t(Y\,|\,z)\ =\ \frac{\int_{x\in\phi^{-1}(z)} p_t(Y\,|\,x)\,p_t(x)\,dx}%
	{\int_{x\in\phi^{-1}(z)} p_t(x)\,dx}\ \neq\ p_s(Y\,|\,z).
\end{equation}
Equality holds for general $p_s$, $p_t$ only if $\phi$ is invertible.'' According to Section~2 of
Johansson et al., $p_s$ and $p_t$ stand for the densities of the covariate $X$ on the `source domain'
and `target domain' respectively. By Theorem~\ref{th:main}, with $\mathcal{G} = \sigma(Z)$, actually
covariate shift holds under the transformation $\phi$ if $\mathcal{G}$ is sufficient for $\mathcal{H} = \sigma(X)$
(in the setting of Johansson et al.). Sufficiency of $\mathcal{G}$ is implied by invertibility of $\phi$. Hence,
Theorem~\ref{th:main} is a more general statement than the one 
by Johansson et al.~\cite{pmlr-v89-johansson19a}.\footnote{%
The derivation of \eqref{eq:Johan} in \cite{pmlr-v89-johansson19a} is somewhat sloppy. 
In Section~2.3 of \cite{pmlr-v89-johansson19a}, the assumption is made for $Z = \phi(X)$ that
`$p(Z)$' is a density. This implies $\int_{x\in\phi^{-1}(z)} p_t(x)\,dx = 0$  which means
that the denominator of the fraction in \eqref{eq:Johan} is zero.}

Under Assumption~\ref{as:setting}, a mapping (representation) 
$T: (\Omega, \mathcal{H}) \to (\Omega_T, \mathcal{H}_T)$ which
is $\mathcal{H}_T$-$\mathcal{H}$-measurable is said to have  `invariant components'
(Gong et al.~\cite{pmlr-v48-gong16}) 
if its distributions under the source and target distributions are the same, i.e.\ if
\begin{equation}\label{eq:invariant}
P[T \in M] \ =\ Q[T \in M], \qquad \text{for all}\ M \in \mathcal{H}_T.
\end{equation}
As $\mathcal{H}$ reflects the covariates, $T$ can be interpreted as a transformation of the covariates that
makes their distributions undistinguishable under the source and target distributions.
As Gong et al.~\cite{pmlr-v48-gong16} noted, \eqref{eq:invariant} alone does not imply
that the posterior probabilities under source and target distributions are the same or at least
similar. He et al.~\cite{he2022domain} therefore defined the notion of `domain invariance' by 
\begin{subequations}
\begin{gather}
P[A_1\,|\,\sigma(T)] = P[A_1\,|\,\mathcal{H}],\quad 
Q[A_1\,|\,\sigma(T)] = Q[A_1\,|\,\mathcal{H}], \quad \text{and} \label{eq:HeEtAl}\\
P[A_i \cap \{T \in M\}] = Q[A_i \cap \{T \in M\}], \quad i \in \{0, 1\}, M \in \mathcal{H}_T.\label{eq:dist}
\end{gather}
\end{subequations}
He et al.~\cite{he2022domain} then observed that \eqref{eq:HeEtAl} and \eqref{eq:dist} together imply
covariate shift with respect to the information set $\mathcal{H}$, i.e.\
$P[A_1\,|\,\mathcal{H}] = Q[A_1\,|\,\mathcal{H}]$.\footnote{%
Actually, \eqref{eq:dist} implies covariate shift with respect to $\sigma(T)$. From this, together with
\eqref{eq:HeEtAl}, follows covariate shift with respect to $\mathcal{H}$. Hence the assumption of 
\eqref{eq:dist} could be replaced by the weaker assumption of having covariate shift with respect to $\sigma(T)$.}
In a sense, the observation by He et al.\ can be considered complementary to Theorem~\ref{th:main} 
because Theorem~\ref{th:main} is about passing on covariate shift from a larger information set
to a smaller one while the observation by He et al.\ is a statement about covariate shift on a smaller
information set implying covariate shift on a larger one.

In unsupervised domain adaptation, the case of source and target distributions where part or all of the
support of the target distribution is not covered by the support of the source distribution
is of great interest \cite{Ben-David2007Representations,pmlr-v89-johansson19a}. In that case,
the target distribution is at least partially singular to the source distribution. Has Theorem~\ref{th:main}
any relevance for this situation? Arguably, representations of the covariates which do not work even
in the plain-vanilla environment of target distributions which are absolutely continuous with respect
to the source distribution, are rather questionable. Hence Theorem~\ref{th:main} may be considered
useful for providing a kind of `fatal flaw' test for representations.

There are situations when covariate shift for a given sub-$\sigma$-algebra $\mathcal{G}$ can be
forced. The most important example of such a situation is sample selection 
(Hein \cite{Hein2009binary}, `Class-Conditional Independent Selection'). Theorem~\ref{th:main} may not 
be relevant then.

However, if the rationale for the assumption of covariate shift is based on causality
considerations (like e.g.\ in Storkey \cite{storkey2009training}), 
the set of covariates associated to the information set $\mathcal{H}$ in
the definition of covariate shift might turn out to be quite large, rendering tedious the task of
estimating the posterior $P[A_1\,|\,\mathcal{H}]$. Theorem~\ref{th:main} provides the condition
under which the size (or dimension) of the set of covariates may be reduced without destroying the
invariance of the posterior class probabilities between the source and arbitrary target distributions. 
This condition does not require any special properties of the
target distributions $Q$ but the harmless requirement of being absolutely continuous with respect 
to the source distribution $P$. Note however that Theorem~\ref{th:main} leaves open the possibility that
the covariate shift property is inherited by a non-sufficient sub-$\sigma$-algebra for some (but not all)
specific target distributions.

If the set of covariates generating $\mathcal{H}$ contains at least one real-valued covariate
which has a Lebesgue-density and  is not independent of $A_1$, then there is no sufficient four-elements 
sub-$\sigma$-algebra $\mathcal{G}$ such that \eqref{eq:simpleShift} holds. For
sufficieny would imply
that the range of the posterior class probability $P[A_1\,|\,\mathcal{H}]$ 
consists of two values only -- which is wrong for probabilities
conditional on continuous random variables. Hence by Theorem~\ref{th:main}
no radically simple approach to class prior 
estimation like \eqref{eq:simple} that would be applicable
under all possible shifts of the covariate distribution is available in this case.

\section{Probing for class prior estimation under covariate shift}
\label{se:Probing}

To the author's best knowledge, there is basically one approach to class prior estimation on the target
dataset under covariate shift: Estimate the posterior probability of the positive class
as a function of the covariates
on the source dataset and then calculate its average
on the target dataset, see \eqref{eq:PCC}. Card and Smith \cite{card2018importance} discuss 
two variants of this approach, one of them 
with and the other without proper calibration of the posterior probabilities -- hence the concept
in principle is the same in both variants.

Under prior probability shift, the simple `confusion matrix method' 
can be deployed to achieve consistent class prior estimates \cite{buck1966comparison,saerens2002adjusting}.
As seen in Sections~\ref{se:loss} and \ref{se:Main}, no similarly simple approach based on merely 
making use of one classifier's output works 
under covariate shift. However, averaging the counting results of a large ensemble of classifiers
trained for a variety of cost-sensitive classification problems would work 
(`probing': Langford and Zadrozny \cite{langford2005estimating}; Tasche \cite{Tasche2022}). 
\subsubsection{Sketch of class prior estimation with probing.}
Define the cost-sensitive (weighted) classification loss (with $0\le t \le 1$) in the
setting of Assumption~\ref{as:setting}: 
\begin{equation*}
L(H,t) \ = \ (1-t)\,P[A_1 \cap (\Omega \setminus H)] + t\,P[A_0\cap H], \quad H \in \mathcal{H}.
\end{equation*}
The probing algorithm adapted to class prior estimation then can be described as follows:
\begin{itemize}
\item[1)] Choose an appropriately `dense' set $0 = t_0 < t_1 < t_2 < \ldots < t_n <1$.
\item[2)] For each $t_i$, $i=1, \ldots, n$, find -- with possibly different approaches -- a nearly
optimal minimising classifier\footnote{%
As before, we identify a set with its indicator function that gives the value 1 on the set
and the value 0 on its complement.
} $H(t_i)$ of $L(H,t_i)$, $H\in\mathcal{H}$.
\item[3)] Let $Z = \sum_{i=1}^n (t_i - t_{i-1})\,\mathbf{1}_{H(t_i)}$.
\item[4)] For all $j$ with $L(\{Z > t_j\}, t_j) < L(H(t_j), t_j)$, replace $H(t_j)$ with
$\{Z > t_j\}$.
\item[5)] Repeat steps 3) and 4) until $L(\{Z > t_j\}, t_j) \ge L(H(t_j), t_j)$ for all $j$.
\item[6)] Calculate $\widehat{q} = \sum_{i=1}^n (t_i - t_{i-1})\,Q[H(t_i)]$
as estimate of the positive class prior probability $Q[A_1]$ under the target distribution.
\end{itemize}

\section{Conclusions}
\label{se:Conc}

We have shown that covariate shift is a fragile notion, in the sense that the
invariance of the posterior class probabilities between source and target distributions
may be lost if the set of covariates on which the posterior probabilities are conditioned is diminished. 
This observation implies that under covariate shift simple estimators of the target prior class probabilities
are infeasible if they are designed 
in the style of the confusion matrix method (adjusted count) which is a popular quantifier
under prior probability shift. 

Valid methods for class prior estimation under covariate shift are the careful 
estimation of the posterior class probabilities conditioned on the full set or 
a sufficient subset of the covariates, 
combined with subsequently averaging
them on the target dataset (probabilistic classify \& count). The application of probing as described in 
Section~\ref{se:Probing} could also prove useful for class prior estimation under covariate shift.
So far, probing for class prior estimation has not yet been thoroughly tested. This could be a 
subject for future research.
\subsubsection{Acknowledgements.} The author is grateful to 
Juan Jos\'e del Coz and Pablo Gonz\'a{}lez for drawing his attention
to the subject of class prior estimation under covariate shift and to four anonymous
reviewers whose comments redounded to significant improvements of the paper.

%
%
%
\bibliographystyle{splncs04}
\bibliography{CovariateShiftNoProblem}

\end{document}